\documentclass[preprint,12pt]{elsarticle}

\usepackage{amssymb}
\usepackage{mathrsfs}
\usepackage{mathrsfs}
\usepackage{mathrsfs}
\usepackage{amsfonts}
\usepackage{amsfonts}
\usepackage{amsfonts}
\usepackage{mathrsfs}
\usepackage{amsmath}
\usepackage{graphicx}
\usepackage{multirow}
\usepackage{caption2}
\usepackage{float}
\usepackage{booktabs}
\usepackage{algorithm}
\usepackage{algorithmic}
\usepackage{subfigure}
\usepackage{color}
\newtheorem{theorem}{Theorem}[section]

\numberwithin{equation}{section}

\newtheorem{definition}[theorem]{Definition}

\newtheorem{lemma}[theorem]{Lemma}

\newtheorem{proposition}[theorem]{Proposition}

\newtheorem{assumption}{Assumption}

\newenvironment{proof}[1][Proof]{\textbf{#1. }}{\ \rule{0.5em}{0.5em}}%

\journal{}

\begin{document}

\begin{frontmatter}



\title{Adaptive Parameter Selection for Kernel Ridge Regression\tnoteref{t1}} \tnotetext[t1]{The research was partially
	supported by the  National Key R\&D Program of China (No.2020YFA0713900) and the
    Natural Science Foundation of China [Grant No  62276209]. Email: sblin1983@gmail.com }


 \author{Shao-Bo Lin}
\address{Center for Intelligent Decision-Making and Machine Learning, School of Management, Xi'an Jiaotong University, Xi'an 710049, China
   }

\begin{abstract}
 This paper focuses on   parameter selection issues of kernel ridge regression (KRR). Due to special spectral properties of KRR, we find that delicate subdivision of the parameter interval shrinks the difference between two successive KRR estimates.
Based on this observation, we develop  
 an early-stopping type parameter selection strategy for KRR according to  the so-called Lepskii-type principle. Theoretical verifications are presented in the framework of learning theory to show that KRR equipped with the proposed parameter selection strategy succeeds in achieving optimal learning rates and adapts to different norms, providing a new record of parameter selection for kernel methods.
\end{abstract}

\end{frontmatter}

\section{Introduction}

Due to  perfect theoretical behaviors in theory \cite{caponnetto2007optimal}, kernel ridge regression (KRR) has been widely used for the regression purpose. Numerous provable  variants such as  Nystr\"{o}m regularization \cite{rudi2015less}, distributed KRR \cite{zhang2015divide}, localized KRR \cite{meister2016optimal} and boosted KRR \cite{lin2019boosted} have been developed to reduce the computational burden and circumvent the saturation \cite{gerfo2008spectral} of KRR. However, theoretical verifications on KRR, as well as its variants,  are built upon the a-priori regularization parameter selection strategy, which is practically infeasible since the a-priori information of the data is generally inaccessible. 

Though the uniqueness of the optimal regularization has been proved in \cite{cucker2002best} and the totally stability studied in \cite{christmann2018total,kohler2022total} illustrated that KRR  performs   stable with respect to the regularization parameter,  posterior choices of regularization parameter to realize the excellent theoretical behaviors of KRR  still  remains open. Three existing approaches for parameter selection of KRR are the hold-out (HO) \cite{caponnetto2007optimal}, discrepancy-type principle (DP) \cite{celisse2021analyzing} and Lepskii-type  principle  (LP) \cite{lu2020balancing,blanchard2019lepskii}.
Numerically,   HO requires a split of the sample set $D$ into training and validation sets; derives a set of KRR estimators via the training set and selects the optimal regularization parameter on the validation set. 
Theoretical optimality of HO was provided in \cite[Chap.7]{gyorfi2002distribution} for expectation and \cite{caponnetto2010cross} for probability. 
However, there are mainly three design flaws of HO. At first, the  validation set is
not involved in the training process, resulting in waste of samples and sub-optimality of
HO in practice. Then, HO generally requires that  the empirical excess risk is an accessible unbiased estimate of the population risk, which prohibits the application of it in deriving parameters for  KRR under the reproducing kernel Hilbert space (RHKS) norm. Finally, as shown in \cite{caponnetto2010cross}, HO is implemented under the assumption that the output is bounded, imposing strong boundedness assumption of the noise. 

Different from HO that is available for almost all least-square regression algorithms, DP  and LP are somewhat exclusive to kernel methods. 
DP, originated from linear inverse problems \cite{engl1996regularization}, devotes to quantifying the fitting error by some computable quantities such as the noise of data \cite{celisse2021analyzing} or complexity of   derived  estimates \cite{lin2019boosted}. Though it is proved to be powerful in the literature of inverse problems \cite{engl1996regularization},  its performance is not, at least in theory, optimal for learning purpose since the derived learning rates   in \cite{celisse2021analyzing} is sub-optimal. LP (also called as the  balancing principle), originally proposed by \cite{lepskii1991problem}, focuses on selecting parameter by bounding  differences of two successive estimates. It was firstly adopted in \cite{de2010adaptive} for the learning purpose to determine the regularization parameter of KRR and then improved in \cite{lu2020balancing} to encode the capacity information of RKHS and \cite{blanchard2019lepskii} to adapt to different norms.  Since LP  does not require the split of  data, it practically performs better than HO \cite{lu2020balancing}.  However, there are also two crucial problems concerning  LP. On one hand, LP  needs recurrently  pairwise comparisons of different KRR estimates, which inevitablely brings additional computational burden. On the other hand,  theoretical results presented in \cite{lu2020balancing}  and \cite{blanchard2019lepskii}   are only near optimal in the sense that there is  at least an additional  logarithmic factor in the learning rates of corresponding KRR.

This paper aims to design an early-stopping type parameter selection strategy based on LP to equip KRR to realize its excellent learning performance in theory. Due to the special spectral property of KRR, we present a close  relation between differences of two successive KRR estimates and the empirical effective dimension and find that   sub-division of the parameter interval plays an important role in quantifying this relation. In particular, our theoretical analysis shows that   the uniform  sub-division of the parameter interval benefits in reflecting the spectral property of KRR, which is beyond the capability of the  coarse sub-division in the logarithmic scale adopted in \cite{lu2020balancing,blanchard2019lepskii}. 
Motivated by this, we propose an implementable and provable early-stopping scheme  with uniform partition of the parameter interval, called as adaptive selection with uniform subdivision (ASUS), to equip KRR. There are two main advantages of ASUS. The first one is that ASUS is actually an early-stopping type parameter selection strategy that succeeds in removing the recurrently pairwise comparisons of LP in the literature \cite{de2010adaptive,lu2020balancing,lepskii1991problem}. The other is that KRR with ASUS is proved to achieve 
  optimal learning rates of KRR, which are better than the rates established for discripancy principle \cite{celisse2021analyzing}, balancing principle \cite{lu2020balancing} and Lepskii principle \cite{blanchard2019lepskii}.

\section{Kernel Ridge Regression and Parameter Selection}

Let   $({\mathcal H}_K, \|\cdot\|_K)$ be the   RKHS induced by a
Mercer kernel $K$ on a  compact  metric space ${\mathcal X}$.  Let $D:=\{(x_{i},y_{i})\}_{i=1}^{|D|}\subset\mathcal X\times\mathcal Y$ with  
$\mathcal Y\subseteq\mathbb R$     be the set of data.
Kernel ridge regression (KRR) \cite{smale2005shannon}  is mathematically defined by
\begin{equation}\label{KRR-global}
    f_{D,\lambda} =\arg\min_{f\in \mathcal{H}_{K}}
    \left\{\frac{1}{|D|}\sum_{(x, y)\in D}(f(x)-y)^2+\lambda\|f\|^2_{K}\right\},
\end{equation}
where $|D|$ denotes the number cardinality of the set $D$.
Since KRR needs to compute the inversion of the $|D|\times|D|$ matrix $\mathbb K+\lambda |D| I$ with $\mathbb K=(K(x_i,x_j))_{i,j=1}^{|D|}$ the kernel matrix,  for a fixed regularization parameter $\lambda$, the storage and training complexities of KRR are $\mathcal O(|D|^2)$ and $\mathcal O(|D|^3)$, respectively.

Theoretical assessments of KRR have been made in large literature \cite{smale2007learning,caponnetto2007optimal,steinwart2009optimal,zhang2015divide,lin2017distributed}, showing that KRR is an excellent  learner in the framework of learning theory \cite{cucker2007learning,steinwart2008support}, in which the samples  are assumed to be drawn identically and independently (i.i.d.) according to an unknown but definite distribution $\rho=\rho(y|x)\times\rho_X$ and the aim   is to build a  tight bound of $\|f_{D,\lambda}-f_\rho\|_\rho$ with  $f_\rho=\int_{\mathcal Y}yd\rho(y|x)$   the  well known regression function  and $\|\cdot\|_\rho$   the norm of the $\rho_X$-square integrable functions spaces $L_{\rho_X}^2$. In some settings such as inverse regression \cite{blanchard2018optimal} and mismatch learning \cite{chang2017distributed}, it also requires  to derive tight bound of $\|f_{D,\lambda}-f_\rho\|_K$  for $f_\rho\in\mathcal H_K$.

To derive tight bounds for  $\|f_{D,\lambda}-f_\rho\|_\rho$  and $\|f_{D,\lambda}-f_\rho\|_K$, the following three  assumptions are standard in   learning theory \cite{cucker2007learning,steinwart2009optimal, blanchard2016convergence,blanchard2019lepskii,lu2019analysis,lu2020balancing}.

\begin{assumption}\label{Assumption:boundedness}
   Assume $\int_{\mathcal Y}
y^2d\rho<\infty$ and
\begin{equation}\label{Boundedness for output}
\int_{\mathcal Y}\left(e^{\frac{|y-f_\rho(x)|}M}-\frac{|y-f_\rho(x)|}M-1\right)d\rho(y|x)\leq
\frac{\gamma^2}{2M^2}, \qquad \forall x\in\mathcal X,
\end{equation}
where $M$ and $\gamma$ are positive constants.
\end{assumption}

Assumption \ref{Assumption:boundedness} is the well known Bernstein noise assumption \cite{caponnetto2007optimal},  which is satisfied for bounded, Gaussian and sub-Gaussian noise. To introduce the second assumption, we introduce the well known integral operator $L_K:L_{\rho_X}^2\rightarrow L_{\rho_X}^2$ (or $\mathcal H_K\rightarrow\mathcal H_K$ if no confusion is made) defined by
$$
           L_Kf=\int_{\mathcal X}f(x)K_xd\rho_X,
$$
where $K_x=K(x,\cdot)$. 

\begin{assumption}\label{Assumption:regularity}
For $r>0$, assume
\begin{equation}\label{regularitycondition}
         f_\rho=L_K^r h_\rho,~~{\rm for~some}  ~ h_\rho\in L_{\rho_X}^2,
\end{equation}
where $L_K^r$ denotes the $r$-th power of $L_K: L_{\rho_X}^2 \to
L_{\rho_X}^2$ as a compact and positive operator.
\end{assumption}

It is easy to see that Assumption \ref{Assumption:regularity} describes the regularity of the regression function $f_\rho$ by showing that larger index $r$ implies better regularity of $f_\rho$. In particular, \eqref{regularitycondition} with $r\geq1/2$ implies $f\in\mathcal H_K$ while $r<1/2$ yields $f\notin\mathcal H_K$.
The third assumption is the capacity assumption measured by the
effective dimension 
$$
        \mathcal{N}(\lambda)={\rm Tr}((\lambda I+L_K)^{-1}L_K),  \qquad \lambda>0.
$$

\begin{assumption}\label{Assumption:effective dimension}
 There exists some $s\in(0,1]$ such that
\begin{equation}\label{assumption on effect}
      \mathcal N(\lambda)\leq C_0\lambda^{-s},
\end{equation}
where $C_0\geq 1$ is  a constant independent of $\lambda$.
\end{assumption}

Based on the above three assumptions, it can be found in 
\cite{caponnetto2007optimal,steinwart2009optimal,lin2017distributed,chang2017distributed} the following lemma.

\begin{lemma}\label{Lemma:Optimal-KRR}
Let $\delta\in(0,1)$. Under Assumptions 1-3 with $0<s\leq 1$ and $\frac12\leq r\leq 1$, if $ \lambda^*=c_0|D|^{-\frac{1}{2r+s}}$, then with confidence at least $1-\delta$, there holds
$$
     |f_{D,\lambda^*}-f_\rho\|_\rho\leq \tilde{C}|D|^{-\frac{r}{2r+s}}\log\frac2\delta,
$$
and
$$
  \|f_{D,\lambda^*}-f_\rho\|_K\leq \tilde{C}|D|^{-\frac{r-1/2}{2r+s}}\log\frac2\delta,
$$
where $c_0,\tilde{C}$  are constants independent of $|D|,\lambda,\delta$.
\end{lemma}

Recalling \cite{caponnetto2007optimal,fischer2020sobolev} that the established learning rates  in Lemma \ref{Lemma:Optimal-KRR} cannot be improved further, KRR is one of the most powerful learning schemes to tackle data satisfying Assumptions 1-3, provided the regularization parameter is appropriately selected.
However, as shown in Lemma \ref{Lemma:Optimal-KRR}, the regularization parameter $\lambda^*$ to achieve the optimal learning rates of KRR depends on the regularity index $r$ and capacity decaying rate $s$ that are difficult to check in practice.   Feasible  strategies to determine the regularization parameters of KRR are thus highly desired. 

Besides HO \cite{gyorfi2002distribution,caponnetto2010cross}, balancing principle, a special realization of LP,  proposed in \cite{de2010adaptive},  is the first strategy, to the best of our knowledge, to  adaptively determine the regularization parameter of KRR. Based on bias-variance analysis, \cite{de2010adaptive} derived  capacity-independent  learning rates for KRR with the proposed balancing principle, which was improved to capacity-dependent in the recent work \cite{lu2020balancing} by introducing the empirical effective dimension 
\begin{equation}\label{Definition-empi-effec}
   \mathcal N_{D}(\lambda):
   ={\rm Tr}[(\lambda|D|I+\mathbb K)^{-1}\mathbb K],\qquad \forall\ \lambda>0,
\end{equation}
where ${\rm Tr}(A)$ denotes the trace of the matrix (or operator) $A$.   It should be mentioned that the balancing principle developed in \cite{lu2020balancing} does not adapt to different norms, i.e., it requires different  strategies to guarantee good performance of KRR in learning functions in different  spaces. This phenomenon was observed by \cite{blanchard2019lepskii}, and a novel realization of LP was presented. To be detailed,  
for  $q\in (0,1)$ and $\lambda_k=q^k$, define 
\begin{eqnarray}\label{Def.WD}
\mathcal W_{D,\lambda}:=
 \frac{1}{|D|\sqrt{ \lambda}}+\left( 1+ \frac{1}{\sqrt{\lambda|D|}}\right) \sqrt{\frac{\max\{\mathcal{N}_D(\lambda),1\}}{|D|}},
\end{eqnarray}
\begin{eqnarray}\label{Def.U}
     \mathcal U_{D,\lambda,\delta} :=   
     \sqrt{\frac{  \log \left(1+8\log \frac{64}{\delta}\frac{1}{\sqrt{\lambda|D|}}\max\{1,\mathcal N_D(\lambda)\}\right)}{\lambda|D|}}   
\end{eqnarray}
and
\begin{equation}\label{Def.K-q}
    K_q:=K_{\delta,D,q}:=\min_{0\leq k\leq -\log_q |D|}\left\{ C_1^*\mathcal U_{D,\lambda_k,\delta}\leq 1/4\right\}  
\end{equation}
with   $C_1^*:=\max\{(\kappa^2+1)/3,2\sqrt{\kappa^2+1}\}$ and $\kappa=\sup_{x\in\mathcal X}\sqrt{K(x,x)}$.

Denote 
\begin{equation}\label{set-q}
    \Lambda_q:=\{\lambda_k=q^k:k=0,1,\dots,K_q\}.
\end{equation}
LP proposed in \cite{blanchard2019lepskii}\footnote{The defined LP in \eqref{Eq:LP_stopping_rule} is slightly different from that in \cite{blanchard2019lepskii}, but their basic ideas are same.}  is defined by
 \begin{eqnarray}\label{Eq:LP_stopping_rule}
        \lambda_{LP}&:=&
       \max 
        \left\{\lambda_k\in\Lambda_q: \|(L_{K,D}+\lambda_k)^\frac12(f_{D,\lambda_{k'}}-f_{D,\lambda_k})\|_K \right.\nonumber\\
       &&\left.\leq  C_{LP} \mathcal W_{D, \lambda_k}\log^3\frac{8}\delta, k'=k+1,\dots,K_q\right\}, 
\end{eqnarray}
where   $\delta\in(0,1)$ denotes the confidence level, $C_{LP}>0$ is a  constant  independent of $|D|,r,s,\lambda_k$, $\delta$ and $L_{K,D}:\mathcal H_K\rightarrow\mathcal H_K$ is the positive operator defined by
\begin{equation}\label{empirical-integral-operator}
     L_{K,D}f:=\frac1{|D|}\sum_{(x,y)\in D}f(x)K_x.
\end{equation}
The following lemma 
 derive from \cite{blanchard2019lepskii}  shows the feasibility of \eqref{Eq:LP_stopping_rule}.      
 
\begin{lemma}\label{Lemma:Optimal-KRR-LP1}
Let $\delta\in(0,1)$. If Assumptions 1-3 hold with $0<s\leq 1$ and $\frac12\leq r\leq 1$,   then with confidence at least $1-\delta$, there holds
$$
     \|f_{D,\lambda_{LP}}-f_\rho\|_\rho\leq \tilde{C}|D|^{-\frac{r}{2r+s}}(\log\log |D|) \log\frac2\delta,
$$
and
$$
  \|f_{D,\lambda_{LP}}-f_\rho\|_K\leq \tilde{C}|D|^{-\frac{r-1/2}{2r+s}}(\log\log |D| )\log\frac2\delta,
$$
where $ \tilde{C}$  is a constant  independent of $|D|,\delta$.
\end{lemma}

Compared with Lemma \ref{Lemma:Optimal-KRR}, the above lemma shows that the regularization parameter determined by \eqref{Eq:LP_stopping_rule} can achieve the optimal learning rates of KRR up to a double logarithmic factor. It can be found in \eqref{Eq:LP_stopping_rule} and Lemma \ref{Lemma:Optimal-KRR-LP1} that there are still two unsettled  issues for LP. From the theoretical perspective, it would be interesting to remove the double logarithmic term in Lemma \ref{Lemma:Optimal-KRR-LP1} so that KRR with LP can achieve the optimal learning rates. From the numerical consideration, it is necessary to remove the recurrently pairwise comparisons in \eqref{Eq:LP_stopping_rule}.  



\section{Adaptive Parameter Selection for KRR}
 
In this section,   we propose an early-stopping type realization of LP  to remove the recurrently pairwise comparisons and prove that the corresponding KRR succeeds in achieving the optimal learning rates in Lemma \ref{Lemma:Optimal-KRR}.
 Before presenting the detailed implementation, we introduce the spectral property of KRR at first to embody the role of subdivision of the parameter interval in the following property.

\begin{proposition}\label{Prop:crucial-stopping-KRR}
If Assumption \ref{Assumption:boundedness} and Assumption \ref{Assumption:regularity} hold with $\frac12\leq r\leq 1$, then for any $\lambda,\lambda'$ satisfying $C_1^*\max\{\mathcal U_{D,\lambda,\delta},\mathcal U_{D,\lambda',\delta} \}\leq 1/4$,
 with confidence $1-\delta$, there holds
\begin{eqnarray}\label{spectral-iteration}
     &&\|(L_{K,D}+\lambda I)^{-1/2}(f_{D,\lambda}-f_{D,\lambda'})\|_K
     \leq  
      2^{r+1/2} \|h_\rho\|_\rho\frac{|\lambda'-\lambda|}{\lambda}\lambda^{r} \nonumber\\
      &+&
      16\sqrt{2}(\kappa M +\gamma)\frac{|\lambda'-\lambda|}{\lambda'} \mathcal W_{D,\lambda}\log^2\frac8\delta.
\end{eqnarray}
\end{proposition}

The proof of the proposition will be postponed in the next section.
Proposition \ref{Prop:crucial-stopping-KRR} quantifies the role of subdivision via the terms $\frac{|\lambda'-\lambda|}{\lambda}$ and $\frac{|\lambda'-\lambda|}{\lambda'}$. If 
 $\lambda_k=q^k$, which has adopted in the literature \cite{de2010adaptive,lu2020balancing,blanchard2019lepskii}, then
 $$
\max \left\{\frac{|\lambda_k-\lambda_{k+1}|}{\lambda_k},
\frac{|\lambda_k-\lambda_{k+1}|}{\lambda_{k+1}}\right\}\leq \frac{1-q}{q},\qquad\forall k=0,1,\dots.
$$
We get from \eqref{spectral-iteration} that 
\begin{equation}\label{stop-motivation-2}
   \|(L_{K,D}+\lambda_{k} I)^{1/2}(f_{D,\lambda_k}-f_{D,\lambda_{k+1}})\|_K
  \leq
   \bar{C}_1(1-q)q^{-1}(\mathcal W_{D,\lambda_k}
   + \lambda_k^{r})\log^2\frac8\delta,
\end{equation}
with $\bar{C}_1:=\max\{2^{r+1/2}\|h_\rho\|_\rho,16\sqrt{2}(\kappa M +\gamma)\}$. However, if we impose more delicate subdivision scheme, i.e., 
 $\lambda_k=\frac1{kb}$ for some $b\in\mathbb N$, then
$$
\max \left\{\frac{|\lambda_k-\lambda_{k+1}|}{\lambda_k},
\frac{|\lambda_k-\lambda_{k+1}|}{\lambda_{k+1}}\right\}\leq \frac1k=b\lambda_k,
$$
which follows 
\begin{equation}\label{stop-motivation-3}
   \|(L_{K,D}+\lambda_k I)^{1/2}(f_{D,\lambda_k}-f_{D,\lambda_{k+1}})\|_K
  \leq
   \bar{C}_1b\lambda_k(\mathcal W_{D,\lambda_k}
   + \lambda_k^{r})\log^2\frac8\delta.
\end{equation}
Comparing \eqref{stop-motivation-3} with \eqref{stop-motivation-2}, there is an additional $\lambda_k$ in the bound of $\|(L_{K,D}+\lambda_k I)^{1/2}(f_{D,\lambda_k}-f_{D,\lambda_{k+1}})\|_K$, showing the power of delicate subdivision of the parameter interval. It should be highlighted that similar results as Proposition \ref{Prop:crucial-stopping-KRR} frequently do not hold for the general spectral regularization algorithms \cite{gerfo2008spectral,guo2017learning,lu2020balancing,blanchard2019lepskii} with filters $g_{\lambda}$ since it is difficult to quantify the difference $g_{\lambda}(L_{K,D})-g_{\lambda'}(L_{K,D})$ directly to embody the role of sub-division. 
  We then propose  the following adaptive selection with uniform subdivision (ASUS) for KRR.

\begin{definition}[Adaptive selection with uniform  subdivision (ASUS)]\label{Definition:adaptive-uni}
 For $b\in\mathbb N$, $\lambda_k=\frac1{bk}$ and $\delta\in(0,1)$, write
\begin{equation}\label{Def.K}
    K^*:=K_{\delta,D,b}:=\min_{0\leq k\leq |D|/b}\left\{ C_1^*\mathcal U_{D,\lambda_k,\delta}\leq 1/4\right\}
\end{equation}
and
\begin{equation}\label{def.candidate-uni}
    \Lambda^{uni}_b:=
    \left\{\lambda_k:= \frac1{bk}:0\leq k\leq K^*    \right\}.
\end{equation}
For $\lambda_k\in\Lambda_b^{uni}$ with $k=K^*,K^*-1\dots,1$,
define $\hat{k}_{uni}$ to be the first $k$ satisfying
\begin{equation}\label{stopping-1}
  \|(L_{K,D}+\lambda_{k-1} I)^{1/2}(f_{\lambda_{k},D}-f_{\lambda_{k-1},D})\|_{K}\geq
C_{US} \mathcal W_{D, \lambda_{k}}\log^2\frac{8}\delta,
\end{equation}
where $C_{US}:=32\sqrt{2}b(\kappa M +\gamma)\}$.
If there is not any $k$ satisfying the above inequality, define $\hat{k}_{uni}=K^*$.
Write $\hat{\lambda}_{uni}=\lambda_{\hat{k}_{uni}}$.
\end{definition}

Different from  LP developed in \cite{de2010adaptive,lu2020balancing,blanchard2019lepskii}, ASUS does not requires recurrently pairwise comparisons and behaves as an early-stopping rule. Furthermore,
 ASUS embodies the role of subdivision by adding  $\lambda_k$ in the right-hand side of the stopping rule. From  Definition \ref{Definition:adaptive-uni}, it follows 
\begin{equation}\label{stopping-2-aaaaaa}
  \|(L_{K,D}+\lambda_k I)^{1/2}(f_{\lambda_{k},D}-f_{\lambda_{k+1},D})\|_{K}<
C_{US}\lambda_k \mathcal W_{D, \lambda_{k}}\log^2\frac{8}\delta,\qquad k\geq\hat{k}_{uni}.
\end{equation}
As discussed in \cite{blanchard2019lepskii}, all the mentioned terms in \eqref{stopping-1} is computable. In fact, the constant $C_{US}$ depends on the noise that can be estimated by using the standard statistical methods in \cite{raskutti2014early,celisse2021analyzing}, $\mathcal W_{D,\lambda_k}$ depends only on the empirical effective dimension $\mathcal N_D(\lambda)$, and 
$$
   \|(L_{K,D}+\lambda_k I)^{1/2}(f_{\lambda_{k},D}-f_{\lambda_{k-1},D})\|_{K}
   =\left(\|f_{D,\lambda_k}-f_{D,\lambda_{k-1}}\|_D^2+\lambda_{k}\|f_{D,\lambda_k}-f_{D,\lambda_{k-1}}\|_K^2\right)^{1/2},
$$
where $\|f\|_D^2=\frac1{|D|}\sum_{i=1}^{|D|}|f(x_i)|^2$ and for $f_{D,\lambda_k}=\sum_{i=1}^{|D|}\alpha^k_iK_{x_i}$ 
$$
   \|f_{D,\lambda_k}-f_{D,\lambda_{k-1}}\|_K^2
   =\sum_{i,j=1}^{|D|}(\alpha_i^k-\alpha_i^{k-1})(\alpha_j^k-\alpha_j^{k-1})K(x_i,x_j).
$$
 
The following theorem presents the optimality of ASUS for KRR. 

\begin{theorem}\label{Theorem:ASUS}
    Let $\delta\in(0,1)$. If Assumptions 1-3 hold with $0<s\leq 1$ and $\frac12\leq r\leq 1$,   then with confidence at least $1-\delta$, there holds
\begin{equation}\label{ASUS-rho}
       \|f_{D,\hat{\lambda}_{uni}}-f_\rho\|_\rho\leq C_1|D|^{-\frac{r}{2r+s}}    \log^4\frac8\delta,
\end{equation}
and
\begin{equation}\label{ASUS-K}
    \|f_{D,\hat{\lambda}_{uni}}-f_\rho\|_K\leq C_1|D|^{-\frac{r-1/2}{2r+s}}   \log^4\frac8\delta,
\end{equation}
where $C_1$  is a constant  independent of $|D|,\lambda,\delta$.
\end{theorem}

Theorem \ref{Theorem:ASUS} shows that, equipped with  ASUS, KRR achieves the optimal learning rates, demonstrating the feasibility and optimality of ASUS.  Without the recurrently pairwise comparisons, ASUS   performs theoretically better than LP in \cite{blanchard2019lepskii} and  \cite{lu2020balancing} via achieving better learning rates. Furthermore, the optimal learning rates presented in Theorem \ref{Theorem:ASUS} also shows that  ASUS theoretically behaves better than the discrepancy principle \cite{celisse2021analyzing} and at least comparable with hold-out \cite{caponnetto2010cross} under the $L^2_{\rho_X}$ metric. It should be highlighted that the reason why we can get such advantages of ASUS is due to the special spectral property of KRR in Proposition \ref{Prop:crucial-stopping-KRR}. It would be interesting and challenging to develop similar parameter selection strategy to equip general spectral regularization algorithms, just as \cite{caponnetto2010cross}, \cite{lu2020balancing}, \cite{blanchard2019lepskii}, \cite{celisse2021analyzing} did for hold-out, balancing principle, Lepskii principle and discrepancy principle, respectively.

\section{Proofs}

We adopt the widely used integral operator approach \cite{smale2004shannon,smale2005shannon,smale2007learning} to prove our main results. Write the sampling operator $S_{D}:\mathcal H_K\rightarrow\mathbb R^{|D|}$ as
$$
         S_{D}f:=\{f(x_i)\}_{(x_i,y_i)\in D}.
$$
Its scaled adjoint $S_{D}^T:\mathbb R^{|D|}\rightarrow \mathcal H_K$ is
$$
       S_{D}^T{\bf c}:=\frac1{|D|}\sum_{(x_i,y_i)\in D}c_iK_{x_i},\qquad {\bf c}\in\mathbb R^{|D|}.
$$
Then, we have $L_{K,D}=S_{D}^TS_{D }$ and $\frac1{|D|}\mathbb K=S_DS_D^T$. It was derived in 
 \cite{smale2005shannon} that KRR possesses the operator representation 
\begin{equation}\label{KRR:operator}
         f_{D,\lambda}=(L_{K,D}+\lambda I)^{-1}S^T_Dy_D,
\end{equation}
where $y_D:=(y_1,\dots,y_{|D|})^T$. The key idea of the integral operator approach is to use operator differences to quantify the generalization error. Define
\begin{eqnarray}
    \mathcal Q_{D,\lambda} &:=&\|(L_K+\lambda I)^{1/2}(L_{K,D}+\lambda I)^{-1/2}\|,\label{Def.QD}\\
    \mathcal P_{D,\lambda}&:=&
	\left\|(L_K+\lambda
	I)^{-1/2}(L_{K,D}f_\rho-S^T_{D}y_D)\right\|_K,\label{Def.PD}\\
 \mathcal S_{D,\lambda}&:=&\|(L_K+\lambda I)^{-1/2}(L_K-L_{K,D})(L_K+\lambda I)^{-1/2}\|.\label{Def.SD}
\end{eqnarray}
The following lemma presenting tight bounds of the above quantities plays a crucial role in our proofs.

\begin{lemma}\label{Lemma:Q}
Let $D$ be a sample drawn independently according to $\rho$ and $0<\delta <1$.  Under Assumption \ref{Assumption:boundedness},  if  $C_1^*\mathcal U_{D,\lambda,\delta}\leq 1/4$, then with confidence at least
	$1-\delta$, there simultaneously holds
\begin{eqnarray}
    \mathcal S_{D,\lambda} &\leq&  C_1^*\left(
		 \frac{\log \max\{1,\mathcal N(\lambda)\}}{\lambda|D|}+ \sqrt{\frac{  \log\max\{1,\mathcal N(\lambda)\}}{\lambda|D|}}\right)\log\frac{8}{\delta},\label{bound-S}\\
   \mathcal Q_{D,\lambda}&\leq& \sqrt{2},\label{bound-q}\\
    \mathcal P_{D,\lambda}   &\leq&   16(\kappa M +\gamma) \left(\frac{1}{|D|\sqrt{ \lambda}}+\left( 1+ \frac{1}{\sqrt{\lambda|D|}}\right) \sqrt{\frac{\max\{\mathcal{N}_D(\lambda),1\}}{|D|}}\right) \log^2\frac8\delta,     \label{Bound-P}    \\
    && (1+4\eta_{\delta/4})^{-1}
		\sqrt{\max\{\mathcal N(\lambda),1\}}
    \leq 
		\sqrt{\max\{\mathcal N_D(\lambda),1\}} \nonumber\\
   &\leq & (1+4\sqrt{\eta_{\delta/4}}\vee\eta_{\delta/4}^2)\sqrt{\max\{\mathcal N(\lambda),1\}}, \label{bound-effective-empri}
\end{eqnarray}
 where $\eta_\delta:=2\log(4/\delta)/\sqrt{\lambda|D|}$.
\end{lemma}

\begin{proof}
The bound in \eqref{bound-S} and \eqref{bound-effective-empri} can be found in \cite{lin2020distributed} and \cite{blanchard2019lepskii}, respectively. To derive \eqref{bound-q}, direct computation yields
\begin{eqnarray*}
      &&(L_K+\lambda I)^{1/2}(L_{K,D}+\lambda I)^{-1}(L_K+\lambda I)^{1/2}\\
         &=& (L_K+\lambda I)^{1/2}[(L_{K,D}+\lambda I)^{-1}-(L_{K}+\lambda I)^{-1}]   (L_K+\lambda I)^{1/2}\\
         &+&
         I
          = I 
           +  (L_K+\lambda I)^{-1/2}(L_K-L_{K,D})(L_K+\lambda I)^{-1/2}\\
           &&(L_K+\lambda I)^{1/2}
          (L_{K,D}+\lambda I)^{-1}(L_K+\lambda I)^{1/2}.
\end{eqnarray*}
We then have from \eqref{Def.SD} that
\begin{eqnarray*}
   && \|(L_K+\lambda I)^{1/2}(L_{K,D}+\lambda I)^{-1}(L_K+\lambda I)^{1/2}\|\\
   &\leq& 1+ \mathcal S_{D,t}\|(L_K+\lambda I)^{1/2}(L_{K,D}+\lambda I)^{-1}(L_K+\lambda I)^{1/2}\|.
\end{eqnarray*}
The only thing remainder is to present a restriction on $\lambda$ so that $\mathcal S_{D,\lambda}<1$. For this purpose, recall  \eqref{bound-S} and we then get that with confidence $1-\delta$, there holds
\begin{eqnarray*}
     \mathcal S_{D,\lambda}
      &\leq& C_1^*\left(
		 \frac{\log \max\{1,\mathcal N(\lambda)\}}{\lambda|D|}+ \sqrt{\frac{  \log\max\{1,\mathcal N(\lambda)\}}{\lambda |D|}}\right)\\
   &\leq&
   2C_1^* \mathcal U_{D,\lambda,\delta}\leq 1/2.
\end{eqnarray*}
We then have $\mathcal Q_{D,\lambda}\leq \sqrt{2}$ and proves \eqref{bound-q}. For \eqref{Bound-P}, it is well known \cite{caponnetto2007optimal} that under Assumption \ref{Assumption:boundedness},  with confidence at least $1-\delta/4$, there holds
$$
   \mathcal P_{D,\lambda}   \leq  2(\kappa M +\gamma) \left(\frac{1}{|D|\sqrt{ \lambda}}
           +\sqrt{\frac{\mathcal{N}(\lambda)}{|D|}}\right) \log\frac8\delta.
$$
This together with \eqref{bound-effective-empri} shows
\begin{eqnarray*}
   \mathcal P_{D,\lambda}
   &\leq&
   2(\kappa M +\gamma)  \left(\frac{1}{|D|\sqrt{ \lambda}}
           +(1+4\eta_{\delta/4})\sqrt{\frac{\max\{\mathcal{N}_D(\lambda),1\}}{|D|}}\right) \log\frac2\delta\\
           &\leq&
       2(\kappa M +\gamma) \left(\frac{1}{|D|\sqrt{ \lambda}}+\left( 1+ \frac{8}{\sqrt{\lambda|D|}}\right) \sqrt{\frac{\max\{\mathcal{N}_D(\lambda),1\}}{|D|}}\right) \log^2\frac8\delta.
\end{eqnarray*}
This proves \eqref{Bound-P} and finishes the proof of Lemma \ref{Lemma:Q}.
\end{proof}

Based on the above lemma, we 
  prove Proposition \ref{Prop:crucial-stopping-KRR} as follows.

\begin{proof}[Proof of Proposition \ref{Prop:crucial-stopping-KRR}]
Since $(L_{K,D}+\lambda I)^{-1}$ and $(L_{K,D}+\lambda' I)^{-1}$ have  same eigenfunctions, we have
$$
     (L_{K,D}+\lambda I)^{-1}(L_{K,D}+\lambda' I)^{-1}=(L_{K,D}+\lambda' I)^{-1}(L_{K,D}+\lambda I)^{-1}.
$$
 Define further
\begin{eqnarray}\label{noise free version}
  f^{\diamond}_{D,\lambda} := (L_{K,D}+\lambda
  I)^{-1}L_{K,D}f_\rho
\end{eqnarray}
to be the  noise free version of $f_{D,\lambda}$. 
Then, it follows from $A^{-1}-B^{-1}=B^{-1}(B-A)A^{-1}$ for positive operators  that
\begin{eqnarray}\label{difference-KRR}
   &&f_{D,\lambda}-f_{D,\lambda'}=((L_{K,D}+\lambda I)^{-1}-(L_{K,D}+\lambda'I)^{-1})S_{D}^Ty_{D} \nonumber\\
   &=&
   (L_{K,D}+\lambda I)^{-1}(\lambda'-\lambda) (L_{K,D}+\lambda' I)^{-1}S_{D}^Ty_{D} \nonumber\\
   &=&
   (L_{K,D}+\lambda I)^{-1}(\lambda'-\lambda)(f_{D,\lambda'}-f^\diamond_{D,\lambda'})
   +(L_{K,D}+\lambda I)^{-1}(\lambda'-\lambda)(f^\diamond_{D,\lambda'}-f_\rho) \nonumber\\
   &+&
   (L_{K,D}+\lambda I)^{-1}(\lambda'-\lambda)f_\rho \nonumber\\
   &=&
   (\lambda'-\lambda)(L_{K,D}+\lambda I)^{-1}
   [(L_{K,D}+\lambda' I)^{-1}(S_{D}^Ty_{D}-L_{K,D}f_\rho)+\lambda'(L_{K,D}+\lambda' I)^{-1}f_\rho+f_\rho] \nonumber\\
   &=&
   (\lambda'-\lambda)(L_{K,D}+\lambda I)^{-1}(L_{K,D}+\lambda' I)^{-1}(S_{D}^Ty_{D}-L_{K,D}f_\rho) \nonumber\\
   &+&
   (\lambda'-\lambda)(I+\lambda'(L_{K,D}+\lambda' I)^{-1})(L_{K,D}+\lambda I)^{-1}f_\rho.
\end{eqnarray}
Therefore, we have
\begin{eqnarray*}
     &&\|(L_{K,D}+\lambda I)^{1/2}(f_{D,\lambda}-f_{D,\lambda'})\|_K\\
     &\leq&
     |\lambda'-\lambda|\|(L_{K,D}+\lambda I)^{-1/2}(L_{K,D}+\lambda' I)^{-1}(S_{D}^Ty_{D}-L_{K,D}f_\rho)\|_K\\
     &+&
     |\lambda'-\lambda|\|(L_{K,D}+\lambda I)^{-1/2}(I+\lambda'(L_{K,D}+\lambda' I)^{-1})f_\rho\|_K\\
     &\leq&
     \frac{|\lambda'-\lambda|}{\lambda'}\mathcal Q_{D,\lambda}\mathcal P_{D,\lambda}+2|\lambda'-\lambda| \|(L_{K,D}+\lambda I)^{-1/2}f_\rho\|_K.
\end{eqnarray*}
But Assumption \ref{Assumption:regularity} together with the Cordes inequality \cite{bhatia2013matrix}
\begin{equation}\label{Cordes-inequality}
    \|A^\tau B^\tau\|\leq \|AB\|^\tau,\qquad 0<\tau\leq 1.
\end{equation}
for positive operators $A,B$ implies
\begin{eqnarray*}
     &&\|(L_{K,D}+\lambda I)^{-1/2}f_\rho\|_K =
     \|(L_{K,D}+\lambda I)^{-1/2}L_K^{r-1/2}\|\|h_\rho\|_\rho\\
    & \leq&
    \|(L_{K,D}+\lambda I)^{r-1}\|\|(L_{K,D}+\lambda I)^{1/2-r}(L_K+\lambda I)^{r-1/2}\|\|h_\rho\|_\rho\leq
      \lambda^{r-1}\mathcal Q^{2r-1}_{D,\lambda}\|h_\rho\|_\rho.
\end{eqnarray*}
Thus, we obtain
$$
  \|(L_{K,D}+\lambda I)^{-1/2}(f_{D,\lambda}-f_{D,\lambda'})\|_K
  \leq
  \frac{|\lambda'-\lambda|}{\lambda'}\mathcal P_{D,\lambda}\mathcal Q_{D,\lambda}+2|\lambda'-\lambda|\lambda^{r-1}\mathcal Q^{2r-1}_{D,\lambda} \|h_\rho\|_\rho.
$$
Due to Lemma \ref{Lemma:Q}, with confidence $1-\delta$,  there holds
\begin{eqnarray*}
     &&\|(L_{K,D}+\lambda I)^{-1/2}(f_{D,\lambda}-f_{D,\lambda'})\|_K
     \leq  
      2^{r+1/2} \|h_\rho\|_\rho|\lambda'-\lambda|\lambda^{r-1}\\
      &+&
      \sqrt{2}\frac{|\lambda'-\lambda|}{\lambda'}16(\kappa M +\gamma) \left(\frac{1}{|D|\sqrt{ \lambda}}+\left( 1+ \frac{1}{\sqrt{\lambda|D|}}\right) \sqrt{\frac{\max\{\mathcal{N}_D(\lambda),1\}}{|D|}}\right) \log^2\frac8\delta.
\end{eqnarray*}
This completes the proof of Proposition \ref{Prop:crucial-stopping-KRR} by noting  \eqref{Def.WD}.
\end{proof}

To prove Theorem \ref{Theorem:ASUS}, we also need three  lemmas. 
This first one is standard in the learning theory literature, we present the proof for the sake of completeness. 
\begin{lemma}\label{Lemma:Error-Decomp-KRR}
Under Assumption \ref{Assumption:boundedness} and Assumption  \ref{Assumption:regularity} with $\frac12\leq r\leq 1$, for any $\lambda\geq  \lambda_{K^*}$, there holds
$$
  \max\{\lambda^{1/2}\|f_{D,\lambda}-f_\rho\|_K, \|f_{D,\lambda}-f_\rho\|_\rho\}\leq
   2^{r-1/2}\lambda^r \|h_\rho\|_\rho
   + 32(\kappa M +\gamma)  \mathcal W_{D,\lambda}\log^2\frac8\delta.
$$
\end{lemma}

\begin{proof}
Let $f^{\diamond}_{D,\lambda}$ be given in \eqref{noise free version}. We have
$$
   \|f_{D,\lambda}-f_\rho\|_\rho\leq  \|f_{D,\lambda}-f^{\diamond}_{D,\lambda}\|_\rho+\|f^{\diamond}_{D,\lambda}-f_\rho\|_\rho 
$$
and
$$
  \lambda^{1/2} \|f_{D,\lambda}-f_\rho\|_K\leq  \lambda^{1/2}\|f_{D,\lambda}-f^{\diamond}_{D,\lambda}\|_K+\lambda^{1/2}\|f^{\diamond}_{D,\lambda}-f_\rho\|_K.
$$
But direct computations yield (e.g. \cite{caponnetto2007optimal,lin2020distributed})
\begin{eqnarray*}
      \max\{\lambda^{1/2}\|f^{\diamond}_{D,\lambda}-f_\rho\|_K, \|f^{\diamond}_{D,\lambda}-f_\rho\|_\rho\}
    \leq\lambda\|(L_K+\lambda I)^{1/2}(L_{K,D}+\lambda I)^{-1}f_\rho\|_K
    \leq \lambda^r\mathcal Q_{D,\lambda}^{2r-1}\|h_\rho\|_\rho
\end{eqnarray*}
and
\begin{eqnarray*}
   &&\max\{\lambda^{1/2}\|f_{D,\lambda}-f^{\diamond}_{D,\lambda}\|_K, \|f_{D,\lambda}-f^{\diamond}_{D,\lambda}\|_\rho\}\\
   &\leq& 
  \|(L_K+\lambda I)^{1/2}(L_{K,D}+\lambda I)^{-1}(L_K+\lambda I)^{1/2}\|\mathcal P_{D,\lambda} 
  =
  \mathcal Q_{D,\lambda}^2\mathcal P_{D,\lambda}. 
\end{eqnarray*}
 Therefore, we obtain 
$$
  \max\{\lambda^{1/2}\|f_{D,\lambda}-f_\rho\|_K,\|f_{D,\lambda}-f_\rho\|_\rho\}\leq
   \lambda^r\mathcal Q_{D,\lambda}^{2r-1}\|h_\rho\|_\rho
   +\mathcal Q_{D,\lambda}^2\mathcal P_{D,\lambda}.
$$
Hence, for any $\lambda\geq \lambda_{K^*}$, we get from Lemma \ref{Lemma:Q} that with confidence $1-\delta$, there holds
$$
 \max\{\lambda^{1/2}\|f_{D,\lambda}-f_\rho\|_K,\|f_{D,\lambda}-f_\rho\|_\rho\}\leq
   2^{r-1/2}\lambda^r \|h_\rho\|_\rho
   + 32(\kappa M +\gamma)  \mathcal W_{D,\lambda}\log^2\frac8\delta.
$$
This completes the proof of Lemma \ref{Lemma:Error-Decomp-KRR}.
\end{proof}

The next lemma presents the feasibility of ASUS when the selected
 $\hat{\lambda}_{uni}$ is small than $\lambda^*$ given in Lemma \eqref{Lemma:Optimal-KRR}.
\begin{lemma}\label{Lemma:large-app}
Let $\delta\in(0,1)$ and $\lambda^*=c_0|D|^{-\frac1{2r+s}}$ be given in Lemma \ref{Lemma:Optimal-KRR}. Under Assumptions 1-3 with $0<s\leq 1$ and $\frac12\leq r\leq 1$, if  
$\hat{\lambda}_{uni}\leq \lambda^*$, then with confidence $1-\delta$, there holds
\begin{equation}\label{Learning-uni-large-rho}
     \|f_{D,\hat{\lambda}_{uni}}-f_\rho\|_\rho\leq \bar{C}_2|D|^{-\frac{r}{2r+s}}\log^2\frac8\delta,
\end{equation}
and
\begin{equation}\label{Learning-uni-large-K}
     \|f_{D,\hat{\lambda}_{uni}}-f_\rho\|_K\leq \bar{C}_2|D|^{-\frac{r-1/2}{2r+s}}\log^2\frac8\delta,
\end{equation} 
where $\bar{C}_2$ is a constant independent of $|D|,\delta$.
\end{lemma}

\begin{proof}
The definition of $\hat{\lambda}_{uni}$ yields 
$$
   C_{US}  {\lambda}_{\hat{k}_{uni}-1} \mathcal W_{D, \hat{\lambda}_{uni}}\log^2\frac{8}\delta\leq
   \|(L_{K,D}+{\lambda}_{\hat{k}_{uni}-1} I)^{1/2}(f_{{\lambda}_{\hat{k}_{uni}},D}-f_{ {\lambda}_{\hat{k}_{uni}-1},D})\|_{K}.
$$
But \eqref{stop-motivation-3} implies that with confidence $1-\delta$, there holds
$$
\|(L_{K,D}+{\lambda}_{\hat{k}_{uni}-1} I)^{1/2}(f_{{\lambda}_{\hat{k}_{uni}},D}-f_{ {\lambda}_{\hat{k}_{uni}-1},D})\|_{K}
  \leq
   \bar{C}_1b{\lambda}_{\hat{k}_{uni}-1}\left(\mathcal W_{D, \hat{\lambda}_{uni}}
   + ({\lambda}_{\hat{k}_{uni}-1})^{r}\right)\log^2\frac{8}\delta.
$$
Recalling the definition of $C_{US}$ and $\bar{C}_1$, we have
$$
   \mathcal W_{D, \hat{\lambda}_{uni}} \leq
   \bar{C}_3({\lambda}_{\hat{k}_{uni}-1})^{r} 
$$
for some $\bar{C}_3$ independent of $D,\delta,\lambda_k$.
Furthermore, it follows from  Lemma \ref{Lemma:Error-Decomp-KRR} that 
\begin{eqnarray*}
    &&\max\{ (\hat{\lambda}_{uni})^{1/2}\|f_{D,\hat{\lambda}_{uni}}-f_\rho\|_K,\|f_{D,\hat{\lambda}_{uni}}-f_\rho\|_\rho \}\\
  &\leq& 2^{r-1/2}(\hat{\lambda}_{uni})^r \|h_\rho\|_\rho
   + 32(\kappa M +\gamma)  \mathcal W_{D,\hat{\lambda}_{uni}}\log^2\frac8\delta 
\end{eqnarray*}
holds with confidence $1-\delta$.
Therefore, with confidence $1-\delta$,
$$
  \max\{ (\hat{\lambda}_{uni})^{1/2}\|f_{D,\hat{\lambda}_{uni}}-f_\rho\|_K,\|f_{D,\hat{\lambda}_{uni}}-f_\rho\|_\rho \}
  \leq \bar{C}_4(\hat{\lambda}_{uni})^{r}   \log^2\frac8\delta,
$$
where $\bar{C}_4$ is a constant independent of $\lambda_k,|D|,\delta$.
Noting $\hat{\lambda}_{uni}\leq \lambda^*$ and $\frac12\leq r\leq 1$, we then have
$$
   \|f_{D,\hat{\lambda}_{uni}}-f_\rho\|_\rho
  \leq \bar{C}_4(\lambda^*)^{r}   \log^2\frac8\delta
  \leq \bar{C}_2|D|^{-\frac{r}{2r+s}} 
$$
and
$$
   \|f_{D,\hat{\lambda}_{uni}}-f_\rho\|_K
  \leq \bar{C}_4(\lambda^*)^{r-1/2}   \log^2\frac8\delta
  \leq \bar{C}_2|D|^{-\frac{r-1/2}{2r+s}},
$$
where $\bar{C}_2$ is a constant independent of $|D|,\delta$.
This  completes the proof of Lemma \ref{Lemma:large-app}.
\end{proof}

In the third lemma, we present an error estimate for $f_{D,\hat{\lambda}_{uni}}$ when $\hat{\lambda}_{uni}\geq \lambda^*$.

\begin{lemma}\label{Lemma:uni-large}
Let $\delta\in(0,1)$. Under Assumptions 1-3 with $0<s\leq 1$ and $\frac12\leq r\leq 1$, if  
$\hat{\lambda}_{uni}>\lambda^*$, then with confidence $1-\delta$, there holds
\begin{equation}\label{Learning-uni-small}
     \|f_{D,\hat{\lambda}_{uni}}-f_\rho\|_\rho\leq \bar{C}_5|D|^{-\frac{r}{2r+s}}\log^4\frac{8}\delta,
\end{equation} 
and
\begin{equation}\label{Learning-uni-small}
     \|f_{D,\hat{\lambda}_{uni}}-f_\rho\|_K\leq \bar{C}_5|D|^{-\frac{r-1/2}{2r+s}}\log^4\frac{18}\delta,
\end{equation} 
where $\bar{C}_5$ is a constant independent of $|D|,\delta$.
\end{lemma}

\begin{proof}
The triangle inequality follows
\begin{equation}\label{tri-rho}
   \|f_{D,\hat{\lambda}_{uni}}-f_\rho\|_\rho
   \leq\|f_{D,\hat{\lambda}_{uni}}-f_{D,\lambda^*}\|_\rho+\|f_{D,\lambda^*}-f_\rho\|_\rho 
\end{equation}
and
\begin{equation}\label{tri-K}
   \|f_{D,\hat{\lambda}_{uni}}-f_\rho\|_K
   \leq\|f_{D,\hat{\lambda}_{uni}}-f_{D,\lambda^*}\|_K+\|f_{D,\lambda^*}-f_\rho\|_K.
\end{equation}
But Lemma \ref{Lemma:Optimal-KRR} shows that
\begin{equation}\label{t0-rho}
       \|f_{D,\lambda^*}-f_\rho\|_\rho\leq \tilde{C}|D|^{-\frac{r}{2r+s}}\log\frac2\delta,\qquad \|f_{D,\lambda^*}-f_\rho\|_K\leq \tilde{C}|D|^{-\frac{r-1/2}{2r+s}}\log\frac2\delta
\end{equation}
holds with confidence $1-\delta$.
Therefore, it suffices to bound $\|f_{D,\hat{\lambda}_{uni}}-f_{D,\lambda^*}\|_\rho$ and $\|f_{D,\hat{\lambda}_{uni}}-f_{D,\lambda^*}\|_K$. Write  $\lambda^*=\lambda_{k^*}\sim \frac1{bk^*}$ for $k^*\in\Lambda_b^{uni}$.
Since $\hat{\lambda}_{uni}=\lambda_{\hat{k}_{uni}}=\frac1{b\hat{k}_{uni}}$ and $\hat{\lambda}_{uni}>\lambda^*$, we have $ \hat{k}_{uni}<k^*$.
It follows from
 the triangle inequality again that
$$
    \|f_{D,\hat{\lambda}_{uni}}-f_{D,\lambda^*}\|_*
    \leq
    \sum_{k=\hat{k}_{uni}}^{k^*-1}\|f_{D,\lambda_k}-f_{D,\lambda_{k+1}}\|_*,
$$
where $\|\cdot\|_*$ denotes either $\|\cdot\|_\rho$ or $\|\cdot\|_K$.
But \eqref{stopping-2-aaaaaa} shows that for any 
  $k=\hat{k}_{uni},\dots,k^*-1$,   there holds
\begin{eqnarray*}
   &&\max\{ \lambda_k^{1/2}\| f_{D,\lambda_k}-f_{D,\lambda_{k+1}}\|_K, \| f_{D,\lambda_k}-f_{D,\lambda_{k+1}}\|_\rho\}\\
   &\leq&
   \mathcal Q_{D,\lambda_k}
   \|(L_{K,D}+\lambda I)^{1/2}( f_{D,\lambda_k}-f_{D,\lambda_{k+1}})\|_K
   \leq C_{US} \mathcal Q_{D,\lambda_k} \lambda_k \mathcal W_{D, \lambda_{k}}\log^2\frac{8}\delta.
\end{eqnarray*}
But Lemma \ref{Lemma:Q} shows that
 with confidence $1-\delta$, there holds
$$
       \max_{k=\hat{k}\dots,k^*-1}\mathcal Q_{D,\lambda_k}
       \leq
   \sqrt{2}.
$$
Hence, for any 
  $k=\hat{k}_{uni},\dots,k^*-1$, with confidence $1-\delta$,  there holds
\begin{eqnarray*}
   \max\{ \lambda_k^{1/2}\| f_{D,\lambda_k}-f_{D,\lambda_{k+1}}\|_K, \| f_{D,\lambda_k}-f_{D,\lambda_{k+1}}\|_\rho\}
   \leq \sqrt{2}C_{US}\lambda_k \mathcal W_{D, \lambda_{k}}\log^2\frac{16}\delta.
\end{eqnarray*}
Due to \eqref{Assumption:effective dimension} and Lemma \ref{Lemma:Q},   
 with confidence $1-\delta$, there holds
\begin{eqnarray}\label{WD-bound-population}
	\mathcal W_{D,\lambda}\leq c_1\left(\frac{1}{\lambda|D|}+\frac{(1+4(1+1/(\lambda|D|)))C_0\lambda^{-s/2}(1+8\sqrt{1/\lambda|D|})}{\sqrt{|D|}}\right)\log^2\frac8\delta,
\end{eqnarray}
where $c_1$ is a constant independent of $D,\lambda_k,\delta$. Under this circumstance, there holds
\begin{eqnarray*}
	&&\mathcal W_{D,\lambda_k}\leq c_2 \sqrt{k^s/|D|}(1+\sqrt{k^{1+s}/|D|})\log^2\frac8\delta,
\end{eqnarray*}
where $c_2:=c_1(1+{\tilde{c}}) (\sqrt{\tilde{c}+1}+(5+4{\tilde{c}})C_0(1+8{\tilde{c}}))^2$.
Then { for any $k=\hat{k}\dots,k^*-1$}, $r\geq 1/2$ yields
\begin{eqnarray}\label{bound-for-mediean1}
&&  { \|f_{D,\hat{\lambda}_{uni}}-f_{D,\lambda^*}\|_\rho}\nonumber\\
     &\leq&
    4c_1c_2b \log^4\frac{8}\delta  \sum_{k=\hat{k}}^{k^*-1}
      k^{-1} \sqrt{k^s/|D|}(1+\sqrt{k^{1+s}/|D|}) \nonumber \\
   &\leq&
    4c_1c_2b(2s+1)\frac{(k^*)^{s/2}}{\sqrt{|D|}}
    \left(1+\frac{(k^*)^{(1+s)/2}}{\sqrt{|D|}}\right)\log^4\frac{8}\delta\nonumber\\
    &\leq&
     \bar{c}_3|D|^{-r/(2r+s)}\log^4\frac{8}\delta,
\end{eqnarray}
where
${c}_3:=4c_1c_2 { \tilde{c}}^{s/2}b(2s+1)(1+{ \tilde{c}}^{(1+s)/2}).$ Plugging (\ref{bound-for-mediean1}) and (\ref{t0-rho}) into (\ref{tri-rho}), we get with confidence $1-\delta$, there holds
$$
          \|f_{\hat{t},D}-f_\rho\|_\rho\leq  {c}_4 |D|^{-r/(2r+s)}\log^4\frac{8}\delta
$$
with $c_4=\max\{{c}_3,\tilde{C}\}.$ The bound of $\|f_{D,\hat{\lambda}_{uni}}-f_{D,\lambda^*}\|_K$ can be derived by using the same method as above.
 This completes the proof of Lemma \ref{Lemma:uni-large}.
\end{proof}



Based on Lemma \ref{Learning-uni-small} and Lemma \ref{Lemma:uni-large}, we can derive Theorem \ref{Theorem:ASUS} directly.

\begin{proof}[Proof of Theorem \ref{Theorem:ASUS}]
Theorem \ref{Theorem:ASUS} is a direct consequence of Lemma \ref{Learning-uni-small} and Lemma \ref{Lemma:uni-large}. This completes the proof of Theorem \ref{Theorem:ASUS}.
\end{proof}

\bibliographystyle{elsarticle-num}
\bibliography{dis}
\end{document}